\documentclass[letterpaper, 10 pt, conference]{ieeeconf}      

\IEEEoverridecommandlockouts                             
\usepackage[utf8]{inputenc}
\usepackage{graphicx}
\usepackage{epsfig}
\usepackage{amsmath}
\usepackage{amssymb} 
\usepackage{longtable}
\usepackage{rotating}
\usepackage{multirow}
\usepackage{array}
\usepackage{graphicx}
\usepackage{mathrsfs}
\usepackage{verbatim}
\usepackage{color}
\usepackage{float}
\usepackage{epstopdf}
\usepackage{url}
\usepackage{graphicx}
\usepackage{mathrsfs}
\usepackage{verbatim}
\usepackage{soul}
\usepackage[noadjust]{cite}
\usepackage{epstopdf}
\usepackage{url}
\usepackage{pgfplots,tikz-3dplot}
\usepackage{amsmath} 
\usepackage{amssymb}  
\usepackage{bbm}
\usepackage{color}
\usepackage{graphics}
\usepackage{graphicx}
\usepackage{epsfig}
\usepackage{epstopdf}
\usepackage{amsfonts}
\usepackage{mathtools}
\usepackage{bm}
\usepackage{fancyhdr}
\usepackage{framed}
\usepackage{subfig}
\usepackage{verbatim}
\usepackage{tikz}
\usepackage{relsize}
\usepackage{mathrsfs}
\usepackage{todonotes}
\usepackage{pgfplots}
\usepackage{grffile}
\usepackage{hyperref}

\usetikzlibrary{arrows}

\usepackage{bigints}
\usepackage{todonotes}
\usepackage[font={footnotesize},labelfont=bf]{caption}

\newtheorem{definition}{\bf{Definition}}
\newtheorem{lemma}{\bf{Lemma}}
\newtheorem{theorem}{\bf{Theorem}}
\newtheorem{remark}{\bf{Remark}}
\newcommand{\bs}{\boldsymbol}

\newtheorem{problem}{\textbf{Problem}}

\newcolumntype{C}[1]{>{\centering\let\newline\\\arraybackslash}m{#1}}

\title{\LARGE \bf
Robust Trajectory Tracking Control for Underactuated \\ Autonomous Underwater Vehicles}

\author{Shahab Heshmati-alamdari,  Alexandros Nikou and Dimos V. Dimarogonas
\thanks{Division of Decision and Control Systems, School of Electrical Engineering and Computer Science, KTH Royal Institute of Technology, Stockholm, Sweden, E-mail: {\tt\small \{shaha, anikou, dimos\}@kth.se}. This work was supported by the H2020 ERC Grant BUCOPHSYS, the Swedish Foundation for Strategic Research (SSF), the Swedish Research Council (VR) and the Knut och Alice Wallenberg Foundation (KAW).  }%
}

\begin{document}
\maketitle \thispagestyle{empty} \pagestyle{empty}

\begin{abstract} 
	Motion control of underwater robotic vehicles is a demanding task with great challenges imposed by external disturbances, model uncertainties and constraints of the operating workspace. Thus, robust motion control is still an open issue for the underwater robotics community. In that sense, this paper addresses the tracking control problem
	of 3D trajectories for underactuated underwater robotic vehicles operating in a constrained workspace including obstacles. In particular, a robust Nonlinear Model Predictive Control (NMPC) scheme is presented for the case of underactuated Autonomous Underwater Vehicles (AUVs) (i.e., vehicles actuated only in surge, heave and yaw). The purpose of the controller is to steer the underactuated AUV to a desired trajectory with guaranteed input and state constraints inside a partially known and dynamic environment where the knowledge of the operating workspace is constantly updated on–line via the vehicle’s on–board sensors. In particular, by considering a ball which covers the volume of the system, obstacle avoidance with any of the detected obstacles is guaranteed, despite the model dynamic uncertainties and the presence of external disturbances representing ocean currents and waves. The proposed feedback control law consists of two parts: an online law which is the outcome of a Finite Horizon Optimal Control Problem (FHOCP) solved for the nominal dynamics; and a state feedback law which is tuned off-line and guarantees that the real trajectories remain bounded in a hyper-tube centered along the nominal trajectories for all times.  Finally, a simulation study verifies the performance and efficiency of the proposed approach. 
\end{abstract}


\section{Introduction}
During the last decades, considerable progress has been made in the field of unmanned marine vehicles, with a significant number of results in a variety of marine activities \cite{heshmati2018cooperative}.  Applications such as ocean forecasting, ecosystem monitoring, underwater inspection of oil/gas pipelines and risers are indicative examples of applications that require the underwater robots to work under various constraints (e.g., communication \cite{Heshmati_AUV2018} energy \cite{Heshmati-Alamdari20155492} and sensing limitations) and increased level of autonomy in order to educing efficiently the various underwater resources \cite{Fossen1} .

A typical marine control problem is trajectory tracking which aims to steer the Autonomous Underwater Vehicle on a reference trajectory\cite{Bechlioulis2017429}.  Classical approaches such as local linearization and input-output decoupling have been used in the past to design motion controllers for underwater vehicles \cite{Fossen2011}. Nevertheless, the aforementioned methods yielded poor closed-loop performance and the results were local, around only certain selected operating points.  Output feedback linearization \cite{moon2018decentralized,subudhi2013static} is an alternative approach  which however is not always possible. Moreover, based on a combined approach involving Lyapunov theory and backstepping, various model-based non-linear controllers have been proposed in the literature \cite{Aguiar20022105,Aguiar20031988,Aguiar20071092,Repoulias20071650,Refsnes2008930}.  More specifically, in \cite{Aguiar20071092},  the authors extend their previous works in \cite{Aguiar20022105,Aguiar20031988}, by proposing a nonlinear adaptive control law that drives an underwater vehicle along a sequence of desired way points.  In the aforementioned control law, employing the backstepping technique, a kinematic controller law is extended to the dynamical model.  In \cite{Repoulias20071650}, the authors proposed a tracking controller by employing the integral backstepping technique. However, the effect of the second hydrodynamic damping force was ignored. 

Dynamic model uncertainties of underwater robotic vehicles have been mainly compensated by employing adaptive control techniques \cite{maalouf2015l1,Antonelli2003221}. For instance, in \cite{Do20041967}, the authors proposed an adaptive control approach for an underwater robotic vehicle in presence of model uncertainties.  However, owing to the sensitivity of the aforementioned controllers on unknown parameters, their capability in a real time experiment remains questionable. Moreover, based on switching control strategies and backstepping techniques, a hybrid parameter adaptation law was presented in \cite{Lapierre200889,Lapierre200992}, where the effects of external disturbances and un--modeled dynamics were not considered.  In addition, for all of the aforementioned motion control strategies, it is not always feasible or straightforward to incorporate input (generalized body forces/torques or thrust) and state (3D obstacles, velocities) constraints into the vehicle's closed-loop motion \cite{Heshmati_ICRA2019}. In that sense, the motion control problem of underwater robots continues to pose considerable challenges to system designers, especially in view of the high-demanding missions envisioned by the marine industry (e.g., ship hull inspection, surveillance of oil platforms, cable installation and tracking, etc.).

On the other hand, the reference/desired trajectory for the underwater robot is usually the result of some path planning techniques \cite{Garau20095, Petres2007331}. The majority of planning techniques are based on off–line optimization schemes, which consider static or quasi–static operational environments. Their output is often a set of way-points or trajectories, which are optimal with respect to energy consumption, while satisfying certain environmental constraints (i.e., known obstacles). However, in real--time missions, the vehicle operates in a partially known and dynamic environment where the knowledge of the operating workspace is constantly updated on--line via the vehicle's on--board sensors (e.g., multi-beam imaging sonars, Onboard vision system, on-line ocean current estimators). In these cases, the underwater vehicle has to re-calculate its path on--line according to possible environmental changes (i.e., new obstacles, other vehicles or humans operating in proximity etc.). More details regarding the open problems on underwater robotics path planning, can be found in \cite{Zeng2015303} and \cite{subramani2016energy}. 

Motivated by the aforementioned considerations, this work presents a robust trajectory tracking control scheme for underactuated Autonomous Underwater Vehicles (AUVs) operating in a constrained workspace including obstacles. In particular, a robust Nonlinear Model Predictive Control (NMPC) scheme is presented for the  underactuated AUVs (i.e., vehicles actuated only in surge, heave and yaw). Various constraints such as: sparse obstacles, workspace boundaries, control input saturation are considered during the control design. The purpose of the controller is to steer the underactuated AUV on a desired trajectory inside a constrained and dynamic workspace. Since the knowledge of the operating workspace is constantly updated online via the vehicle’s on–board sensors, the robot re-calculates its path online if the updated environmental changes (i.e., new detected obstacles) are in conflict with the reference trajectory.  In particular, by considering a ball which covers the volume of the system, obstacle avoidance with any of the detected obstacles is guaranteed, despite the model dynamic uncertainties and the presence of external disturbances representing ocean currents and waves. The proposed feedback control law consists of two parts: an online law which is the outcome of a Finite Horizon Optimal Control Problem (FHOCP) solved for the nominal dynamics; and a state feedback law which is tuned off-line and guarantees that the real trajectories remain bounded in a hyper-tube centered along the nominal trajectories for all times. The volume of the hyper-tube depends on the upper bound of the disturbances as well as bounds of derivatives of the dynamics. The closed-loop system has analytically guaranteed stability and convergence properties. Finally, a simulation study verifies the performance and efficiency of the proposed approach. 


\section{Notation and Background} \label{sec:notation_preliminaries}
In this work, the vectors are denoted with lower bold letters whereas the matrices by capital bold letters.
Define by $\mathbb{N}$ and $\mathbb{R}$ the sets of positive integers and real numbers, respectively. Given a set $\mathcal{S}$, denote by $|\mathcal{S}|$ and $\mathcal{S}^n \coloneqq \mathcal{S} \times \dots \times \mathcal{S}$ its cardinality and its $n$-fold Cartesian product. Given a vector $\bs{z} \in \mathbb{R}^{n}$ define by: $$\|\bs{z}\|_{2} \coloneqq \sqrt{\bs{z}^\top \bs{z}}, \ \ \|\bs{z}\|_{\scriptscriptstyle \infty} \coloneqq \max_{i = 1, \dots, n}{|\bs{z}_i|},  \ \ \|\bs{z}\|_{P} \coloneqq \sqrt{\bs{z}^\top \bs{P} \bs{z}},$$ its Euclidean, infinite and $\bs{P}$-weighted norm, respectively, with $\bs{P} \ge 0$.  The notation  $\lambda_{\scriptscriptstyle \min}(\bs{P})$ stands for the minimum absolute value of the real part of the eigenvalues of $\bs{P} \in \mathbb{R}^{n \times n}$; $0_{m \times n} \in \mathbb{R}^{m \times n}$ and $I_n \in \mathbb{R}^{n \times n}$ stand for the $m \times n$ matrix with all entries zeros and the identity matrix, respectively. The notation ${\rm diag}\{P_1, \dots, P_n\}$ stands for the block diagonal matrix with the matrices $P_1$, $\dots, P_n$ in the main diagonal; $$\mathcal{B}(\bs{c}, r) \coloneqq \left\{\bs{x} \in \mathbb{R}^{n} : \|\bs{x}-\bs{c}\|_{2} \le r \right\},$$ stands for a ball in $\mathbb{R}^{n}$ with center and radius $\bs{c} \in \mathbb{R}^{n}$, $r > 0$, respectively. Given~sets~$\mathcal{S}_1$, $\mathcal{Z}$~$\subseteq \mathbb{R}^n$, $\mathcal{S}_2 \subseteq \mathbb{R}^{m}$~and~matrix $\bs{P} \in \mathbb{R}^{n \times m}$,~the \emph{Minkowski addition}, the~\emph{Pontryagin~difference} and the \emph{matrix-set multiplication} are respectively defined by:
\begin{align*}
\mathcal{S}_1 \oplus \mathcal{Z} & \coloneqq \{\bs{s}_1 + \bs{z} : \bs{s}_1 \in \mathcal{S}_1, \bs{z} \in \mathcal{Z}\}, \\
\mathcal{S}_1 \ominus \mathcal{Z} & \coloneqq \{\bs{s}_1 \in \mathcal{S}_{1} : \bs{s}_1+\bs{z} \in \mathcal{S}_1, \forall \bs{z} \in \mathcal{Z}\}, \\
P \circ \mathcal{S}_{2} & \coloneqq \{\bs{P}\bs{s}, \bs{s} \in \mathcal{S}_{2} \}.
\end{align*}
\begin{definition} \label{def:RPI_set} \cite{khalil_nonlinear_systems}
Consider a dynamical system: $$\dot{\bs{x}} = f(\bs{x},\bs{u},\bs{d}), \ \bs{x} \in \mathcal{X}, \ \bs{u} \in \mathcal{U}, \ \bs{d} \in \mathcal{D}, $$ with initial condition $\bs{x}(0) \in \mathcal{X}$ and external disturbances $\bs{d} \in \mathcal{D}$. A set $\mathcal{X}' \subseteq \mathcal{X}$ is a \emph{Robust Control Invariant (RCI) set} for the system, if there exists a feedback control law $\bs{u} \coloneqq \kappa(\bs{x}) \in \mathcal{U}$, such that for all $\bs{x}(0) \in \mathcal{X}'$ and for all $\bs{d} \in \mathcal{D}$ it holds that $\bs{x}(t) \in \mathcal{X}'$ for all $t \ge 0$, along every solution $\bs{x}(t)$ of the closed-loop system.
\end{definition}
\begin{figure}[t!]
	\centering
	\includegraphics[scale=0.26]{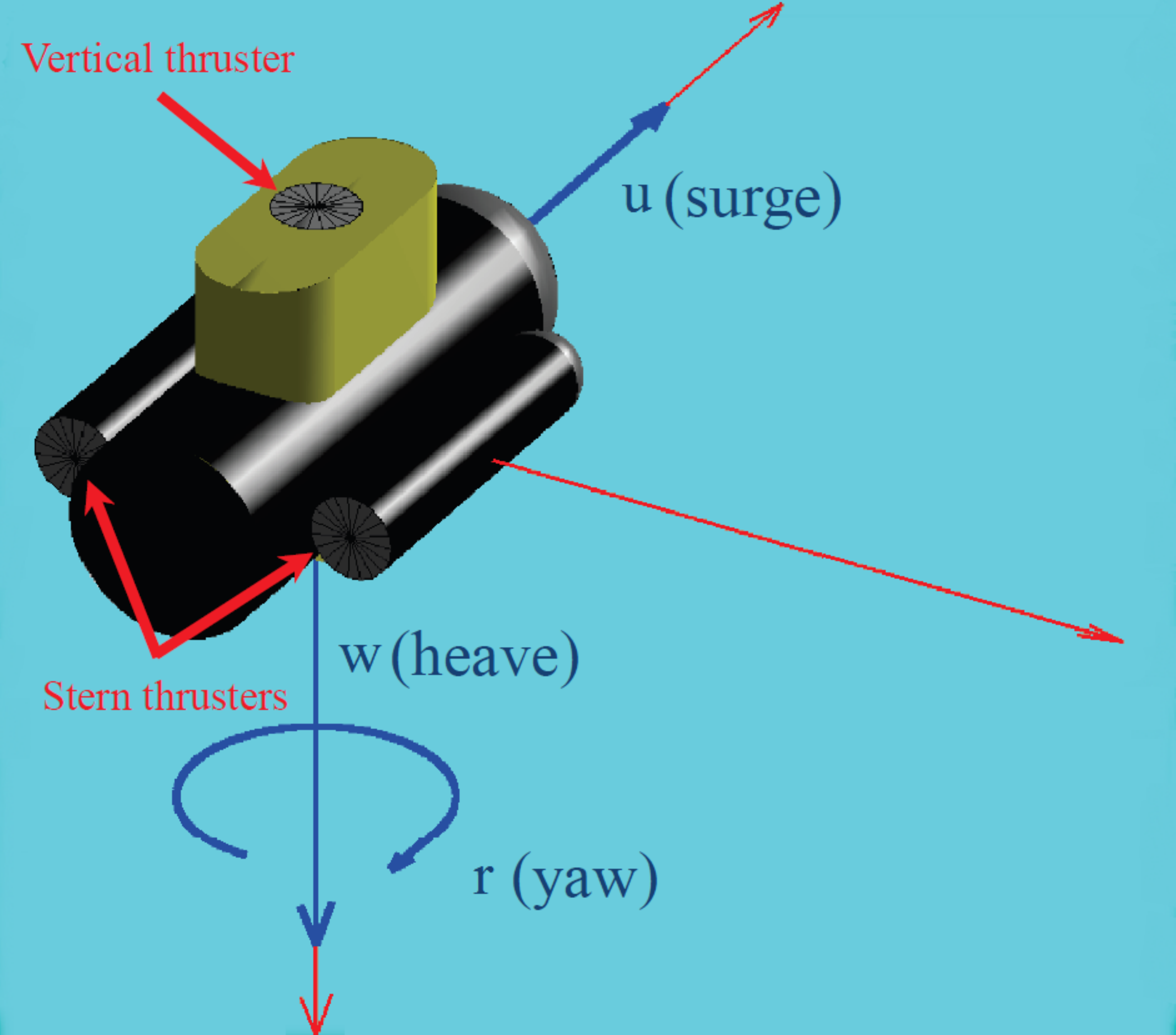}	
	\caption{The underactuated underwater vehicle. Blue color indicates the actuated degrees of freedom.}\label{nonholonomic}\vspace{-4mm}
\end{figure}
\section{Problem Statement} \label{Section:problem}
In this section, the overall problem is formulated. Initially, the nominal model of the under-actuated underwater vehicle and  the constraints that must be fulfilled are presented.
\subsection{Mathematical Modeling}
The prior step before analyzing the proposed methodology is the presentation of the preliminary aspects of the modeling of underwater vehicles.  The pose vector of the vehicle with respect to (w.r.t.) the inertial frame $\mathcal{I}$ is denoted by $\boldsymbol{\eta}=\left[\boldsymbol{\eta^{T}_1}~\boldsymbol{\eta^{T}_2} \right]^T\in\mathbb{R}^{6}$ including the position (i.e., $\boldsymbol{\eta_1}=\left[x~y~z\right]^T$) and orientation (i.e., $\boldsymbol{\eta_2}=\left[\phi~\theta~\psi\right]^T$)
vectors. The $\boldsymbol{v}=\left[\boldsymbol{v^{T}_1}~\boldsymbol{v^{T}_2} \right]^T\in\mathbb{R}^{6}$ is the velocity vector of the vehicle expressed in body fixed frame $\mathcal{V}$ and includes the linear (i.e., $\boldsymbol{v_1}=\left[u~v~w\right]^T$) and angular (i.e., $\boldsymbol{v_2}=\left[p~q~r\right]^T$) velocity vectors where the components have been named according to SNAME \cite{SNAME} as surge, sway, heave, roll, pitch and yaw respectively (Fig.\ref{nonholonomic}). Without loss of generality, according to the standard underwater vehicles' modeling properties \cite{Fossen2}, assuming that the vehicle is operating at relative low speeds, the dynamic equations of the vehicle can be given as:
\begin{subequations}\label{eq:AUV_model}
	\begin{gather}%
	{\boldsymbol{\dot{\eta}}}=\boldsymbol{J}\left( \boldsymbol{\eta}\right)  \boldsymbol{v}\\
	\boldsymbol{M}{\boldsymbol{\dot{v}}}\!+\! \boldsymbol{C}\left(  \boldsymbol{v}\right)  \boldsymbol{v}\!+\! \boldsymbol{D}(\bs{v}) \boldsymbol{v}\! +\! \boldsymbol{g}\left(  \boldsymbol{\eta}\right)
	=\boldsymbol{\tau }
	\end{gather}
\end{subequations}
where:
\begin{itemize}
	\item $\boldsymbol{\tau}=\left[X,~Y,~Z,~K,~M,~N\right]^T\in\mathbb{R}^{6}$ is the total propulsion force/torque vector (i.e., the body forces and torques generated by the thrusters) applied on the vehicle and expressed in body-fixed frame $\mathcal{V}$;
	
	\item $\boldsymbol{M}=\boldsymbol{M}_{RB}+\boldsymbol{M}_{A}$, where $\boldsymbol{M}_{RB}\in\mathbb{R}^{6\times6}$ and $\boldsymbol{M}_{A}\in\mathbb{R}^{6\times6}$  are the inertia matrix for the rigid body and added mass respectively;
	
	\item $\boldsymbol{C}\left(  \boldsymbol{v}\right) = \boldsymbol{C}_{RB}\left(  \boldsymbol{v}\right)  +\boldsymbol{C}_{A}\left(  \boldsymbol{v}\right)
	~$, where $\boldsymbol{C}_{RB}\left(  \boldsymbol{v}\right)\in\mathbb{R}^{6\times6}$ and $\boldsymbol{C}_{A}\left(  \boldsymbol{v}\right)\in\mathbb{R}^{6\times6}$ are the coriolis and centripetal matrix for the rigid body and added mass respectively;
	
	\item $\boldsymbol{D}\left(\boldsymbol{v}\right)  = \boldsymbol{D}_{quad}\left(  \boldsymbol{v}\right)  +\boldsymbol{D}_{lin}\left(  \boldsymbol{v}\right)
	~$, where $\boldsymbol{D}_{quad}\left(  \boldsymbol{v}\right)\in\mathbb{R}^{6\times6}$ and $\boldsymbol{D}_{lin}\left(  \boldsymbol{v}\right)\in\mathbb{R}^{6\times6}$ are the quadratic and linear drag matrix respectively;
	
	\item $\boldsymbol{g}\left(\boldsymbol{\eta}\right)\in\mathbb{R}^{6}$ is the hydrostatic restoring force vector;
	
	\item $\boldsymbol{J}\left(  \boldsymbol{\eta}\right)=\left[\begin{array}{cc}
	\boldsymbol{J}_1\left(\boldsymbol{\eta_2}\right) & \bs{0}_{3\times 3}\\
	\bs{0}_{3\times 3} & \boldsymbol{J}_2\left(\boldsymbol{\eta_2}\right)
	\end{array}\right]$ is the Jacobian matrix transforming the
	velocities from the body-fixed ($\mathcal{V}$) to the inertial ($\mathcal{I}$) frame, in which $\boldsymbol{J}_1\left(\boldsymbol{\eta_2}\right)\in SO(3)$ is the well known rotation matrix and $\boldsymbol{J}_2\left(\boldsymbol{\eta_2}\right)\in\mathbb{R}^{3\times3}$ denotes the lumped transformation matrix \cite{Fossen2};
\end{itemize}

We assume that the vehicle is equipped with three thrusters, which are effective only in surge, heave and yaw motion (Fig.\ref{nonholonomic}), meaning that the vehicle is under-actuated along the sway axis. Moreover, we assume that the vehicle is designed with meta-centric restoring forces in order to regulate roll and pitch angles. Thus, the angles \(\phi\), \(\theta\) and angular velocities \(p\) and \(q\) are negligible and we can consider them to be equal to zero. Thus, from now on, we denote by $\bs{x}=[x,y,z,\psi]^\top$ and $\bs{v}=[u,w,r]^\top$ the state and the control input of the system respectively. The vehicle is symmetric around the \(x\) - \(z\) plane and close to symmetric around the \(y\) - \(z\) plane. Therefore, we can safely assume that motions in heave, roll and pitch are decoupled \cite{Fossen2,Heshmati-Alamdari20143826}. 
{{Because the vehicle is operating at relatively low speeds, the coupling effects are considered to be negligible. Due to the above assumptions and the relative low speed of the vehicle, we consider the vehicle's motion equations, which given as follows:
\begin{align}
& \dot{\bs{x}}=f(\bs{x})\bs{v}+ g(\bs{x},v) \label{eq2} \end{align}
where $\dot{\bs{x}}=[\dot{x},~\dot{y},~\dot{z},~\dot{\psi}]^\top$ and:
\begin{align*}
&f(\bs{x})\bs{v} =
\begin{bmatrix}\cos(\psi) & 0 & 0\\ \sin(\psi)& 0 &0 \\0 & 1& 0\\0&0&1\end{bmatrix}
\begin{bmatrix}u\\ w\\r\end{bmatrix},~g(\bs{x},v)=
\begin{bmatrix}-\sin(\psi) \\ \cos(\psi)\\0\\0
\end{bmatrix}v
\end{align*}
Notice that the $v$ indicates the vehicle velocity in the sway direction.}} In \cite{mikacdc11}, using Input-to-State Stability (ISS) framework, it was shown that for any vehicle described
by \eqref{eq2} and for any bounded control input $[u , r]$ the
velocity around the sway direction $v$ can be seen as a bounded
perturbation with upper bound  $||v||\leq \bar{v}$. Consequently, this point is an equilibrium of
the kinematic system of \eqref{eq2}. Note that throughout this
paper the notation $(\ \bar{\cdot}\ )$ will denote the upper bound
for each of the variables. Therefore we consider the system:
\begin{align}
\dot{\bs{x}}=f(\bs{x})\bs{v} \label{eq3} \end{align}
as the nominal kinematic system of the underwater vehicle, while the
function $g(\cdot)$ is considered as a bounded inner disturbance
of the system that vanishes at the origin. Also, $g(\bs{x},v)\in
\Gamma \subset \mathbb{R}^4$ with $\Gamma$ being a compact set,
such that:
\begin{align}
||g(\bs{x},v)||\leq \bar{\gamma} \quad \text{with} \quad \bar{\gamma}\triangleq\bar{v}
\end{align}
\begin{figure}[t!]
	\centering
	\includegraphics[scale=0.26]{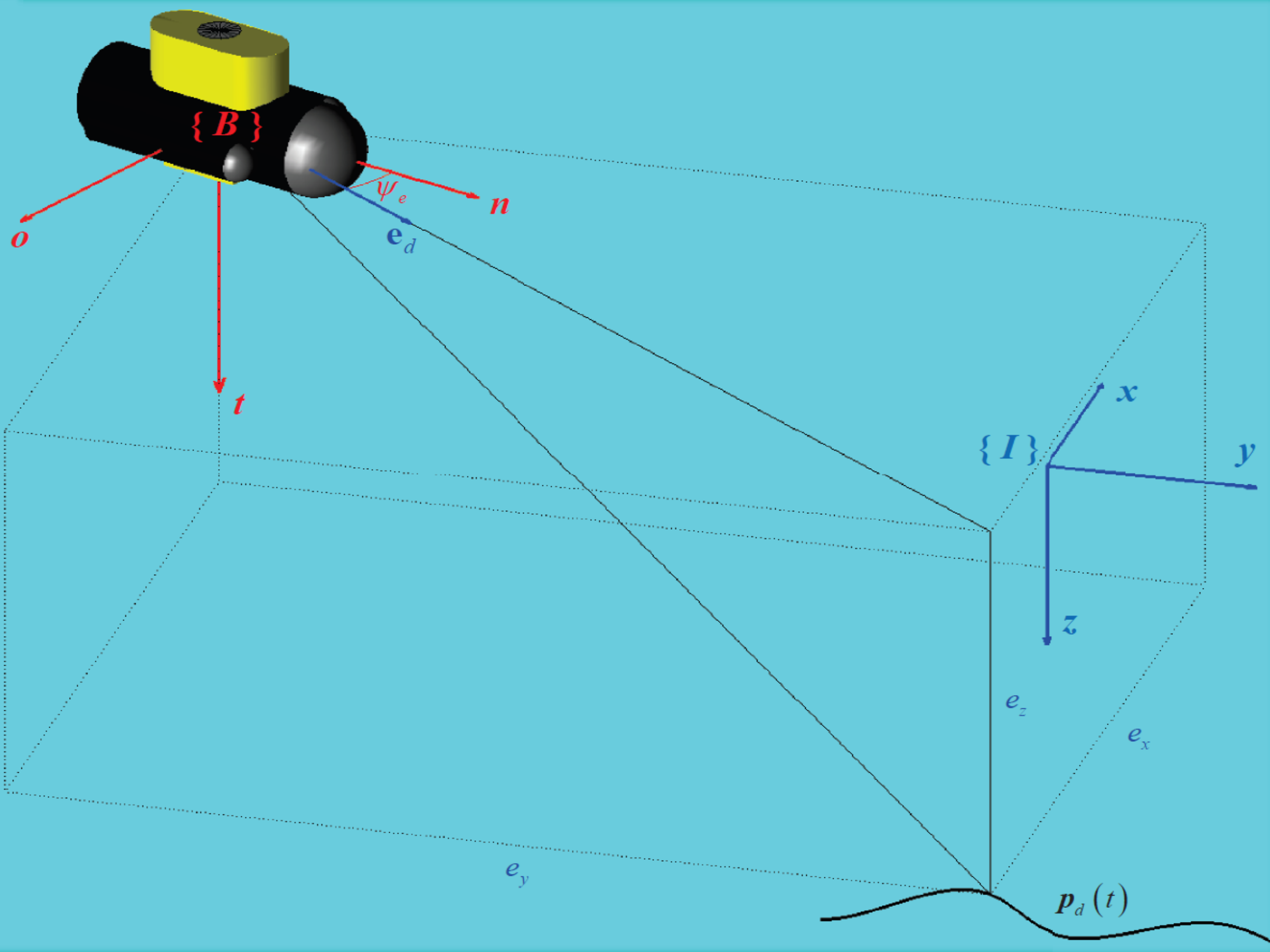}	
	\caption{Graphical illustration of the error and the sea current definitions.}\label{fig1}\vspace{-4mm}
\end{figure}
The robot moves under the influence of an irrotational current which behaves as an external disturbance. The current has
components with respect to the $x$, $y$ and $z$ axes, denoted by $\delta_x$, $\delta_y$ and $\delta_z$, respectively. Also it
is assumed to have a slowly-varying velocity $\delta_c$ which is upper bounded by $||\delta_{c}||\leq \bar{\delta}_c$ and it has direction $\beta$
in the $x-y$ plane and $\alpha$ with respect to the $z$-axis of the global frame, see Fig. \ref{fig1}. In particular we set
$\delta=[\delta_{x},\delta_{y},\delta_{z},0]^\top
\in \Delta\subset\mathbb{R}^4$ with $\Delta$ being a compact set,
where:
\begin{align}
&\delta_{x}\triangleq \delta_{c} \cos(\beta) \sin(\alpha)  \nonumber\\
&\delta_{y}\triangleq \delta_{c}  \sin(\beta)  \sin(\alpha) \nonumber\\
&\delta_{z}\triangleq \delta_{c} \cos(\alpha)   \label{eq10}
\end{align}
It is easy to show that there exists a $\bar{\delta} >0$ such that  $||\delta||\leq \bar{\delta}$. Taking into consideration
the aforementioned disturbances that affect the vehicle, we can model the perturbed system as follows:\vspace{-1mm}
\begin{align}
\dot{\bs{x}}=f(\bs{x})\bs{v}+\bs{\omega}\label{11}
\end{align}
with $\bs{\omega}=[\omega_1,~\omega_2,~\omega_3,~0]^\top \coloneqq g(\bs{x},v)+\delta \in \Omega \subset \mathbb{R}^4
$ as the result of adding the inner and external disturbances of
the system and $\Omega$ is a compact set with
$\Omega=\Delta\oplus\Gamma$. Since the sets
$\Delta$ and $\Gamma$ are compact, we have that $\Omega$ is also
a compact set, and thus we have:
\begin{align}
\Omega=\{ \bs{\omega}(t)\in \mathbb{R}^4 : || \bs{\omega}(t) ||_2  \leq \bar{\omega} \}\label{Wset}
\end{align}
with
$\bar{\omega}\triangleq\bar{\delta}+\bar{\gamma}$ .

\subsection{Geometry of the Workspace}\label{Geometry}
We consider that the underwater vehicle operates inside a workspace $\mathcal{W}\subset \mathbb{R}^3$ with boundary  $\partial\mathcal{W}$ and scattered obstacles located within it. Without loss of the generality, the robot and the obstacles are modeled by spheres (i.e., we
adopt the spherical world representation \cite{Koditschek1990412}). 
Let $\mathcal{B}(\boldsymbol{\eta}_1, \bar{r})$  be a closed ball that covers the whole vehicle volume (main body and additional equipments). Moreover, let $\mathcal{B}(\boldsymbol{\eta}_1, \bar{R})$ with $\bar{R}> \bar{r} $ be a sensing area where the robot can perceive and update its knowledge of the workspace (i.e., the obstacle locations) using its on-board sensors. Furthermore, the $\mathcal{M}$ static obstacles within the workspace are defined as closed balls described by $\pi_m=\mathcal{B}(\bs{p}_{\pi_m},r_{\pi_m}),~ m\in\{1,\ldots,\mathcal{M}\}$, where $\bs{p}_{\pi_m}\in \mathbb{R}^3$ is the center and $r_{\pi_m}>0$ the radius of the obstacle $\pi_m$. Additionally, based on the property of spherical world \cite{Koditschek1990412}, for each pair of obstacles $m,m'$ the following inequality holds:\vspace{0mm}
\begin{align}
|| \pi_{m}-\pi_{m'}|| > 2\bar{r} + r_{\pi_m}+r_{\pi_m'}
\label{obstacle_dis}
\end{align} 
which intuitively means that the obstacles $m$ and $m'$ are disjoint in such a way that the entire volume of the vehicle can pass through the free space between them. Therefore, there exists a feasible trajectory $\bs{\eta}(t)$ for the vehicle such as: \vspace{-3mm}
\begin{align}
\mathcal{B}&(\boldsymbol{\eta}_1(t), \bar{r})\!\cap\! \{ \mathcal{B}(\bs{p}_{\pi_m},r_{\pi_m}\!) \!\cup\! \partial \mathcal{W} \}\! =\! \emptyset,\nonumber\\
&\hspace{35mm} \!\forall t \geq 0, \forall m\! \in\!\{1,\!\ldots\!,\mathcal{M}\!\}\!\label{obstacle_eq}
\end{align}
\subsection{Constraints}
\subsubsection{State Constraints}
as already stated, the robot should be able to avoid the newly detected obstacles which may had been unknown to the off-line trajectory planner. This requirement is captured by the state constraint set $X$ of the system, given by: \vspace{-0mm}
\begin{equation}
\bs{x}(t)\in {X} \subset \mathbb{R}^4 \label{eq4}
\end{equation}\vspace{-1mm}
which in view of \eqref{obstacle_eq} can be defined as:
\begin{align}
X&\coloneqq \Big\{ \bs{x} \in \mathbb{R}^4:\mathcal{B}(\boldsymbol{\eta}_1, \bar{r})\!\cap\! \{ \mathcal{B}(\bs{p}_{\pi_m},r_{\pi_m}) \!\cup\! \partial \mathcal{W} \}\! =\! \emptyset,\nonumber\\
&\hspace{40mm}  \forall m\! \in\!\{1,\!\ldots\!,\mathcal{M}\!\}\! \Big\}\label{feasible set}
\end{align}
{{\subsubsection{Input Constraints}
for the needs of several common underwater tasks (e.g., seabed inspection, mosaicking), the vehicle is required to move with relatively low speeds with upper bound denoted by the velocity vector $\bar{\bs{v}} = [\bar{u}\;\bar{w}\;\bar{r}]^\top$ which defines the  control constraint set ${V}$ as follows:
\begin{equation}
\bs{v}(t) \in V \subseteq \mathbb{R}^3
\label{eq6}
\end{equation}
These constraints are of the form $|{u}(t)|\leq \bar{u}$ , $|{w}(t)|\leq \bar{w}$,
$|r(t)|\leq \bar{r}$, that is:
\begin{align}
V\coloneqq \Big\{ \bs{v} \in \mathbb{R}^3 : \| \bs{v} \|_\infty \leq  \bar{V}   \Big\}\label{Vbar}
\end{align}
with  $\bar{V}=(\bar{u}^2+\bar{w}^2+\bar{r}^2)^{\frac{1}{2}}$. }}
\subsection{Control Objective}\label{Sec:problem}
Let $\bs{p}_d(t) = [{x}_d (t) , {y}_d (t) , {z}_d (t)]^T$ denote a smooth desired trajectory with bounded time derivatives. The controller's objective is to guide the underactuated underwater vehicle to track the aforementioned desired trajectory $\bs{p}_d(t)$ while guarantying various state and input limitations remain satisfied, despite the presence of exogenous disturbances representing ocean currents and waves. Moreover, as stated previously, the underwater robots in most cases are operating in a partially known and dynamic environment where the knowledge of the workspace (e.g., obstacles positions) is updated as the robot moves through the workspace. Since the desired trajectory for the underwater vehicle is usually derived from an offline path planning technique, it may not be particularly feasible in terms of possible environmental changes. 
Thus, the controller must have the capability to be flexible with respect to changes on the knowledge of the operating workspace. Hence, the problem of this paper can be defined as follows:
\begin{problem}\label{prob1}(Robust Tracking Control for an Autonomous Underactuated Underwater Vehicle):
Consider an Underactuated Autonomous Underwater Vehicle described by \eqref{11} operating in a workspace $\mathcal{W} \subset \mathbb{R}^{3}$ with state, input constraints as well as disturbances imposed by the sets ${X}$, ${V}$ and $\Omega$ as in \eqref{eq4}, \eqref{eq6} and \eqref{Wset}, respectively. Consider also that the robot and the obstacles are all modeled according to the spherical world representation\footnote{as described in section-\ref{Geometry} }and the knowledge of the operating workspace $\mathcal{W}$ is constantly updated via the vehicle’s on–board sensors inside a sensing region defined by $\mathcal{B}(\boldsymbol{\eta}_1, \bar{R})$. Given a desired trajectory  $\bs{p}_d(t) = [{x}_d (t) , {y}_d (t) , {z}_d (t)]^T$, design a feedback control law $\bs{v} = \kappa(\bs{x})$ such that the desired trajectory $\bs{p}_d(t)$ is tracked with guaranteed input and state constraints while avoiding any collision with the obstacles, despite the presence of exogenous disturbances representing ocean currents and waves.
\end{problem}

\section{Main Results} \label{SEC:MAIN}
In this section we present the methodology proposed in order to formulate the solution of Problem-\ref{prob1} defined in Section-\ref{Sec:problem}. In particular, a Nonlinear Model Predictive Control (NMPC) framework \cite{michalska_1993, frank_1998_quasi_infinite, mayne_2000_nmpc, grune2016nonlinear} is utilized, and a relevant robust NMPC analysis,  the so-called tube-based approach is provided here for the trajectory tracking problem for underactuated systems in presence of disturbances. The proposed feedback control law consists of two parts: an on-line control law which is the outcome of Finite Horizon Optimal Control Problem (FHOCP) for the nominal system dynamics and a state feedback law which guarantees that the real system trajectories always lie within a hyper-tube centered along the nominal trajectories. First we begin by defining the error states and the corresponding transformed constraints. 

\subsection{Error Definitions} \label{sec:errors}
Given the desired trajectory $\bs{p}_d(t) = [{x}_d (t) , {y}_d (t) , {z}_d (t)]^T$, let us define the position errors:
\begin{align}
e_x (t)\!  =\!x\!-\!x_d(t),~ e_y (t)\!=\! y\! -\! y_d(t),~ e_z(t) \!=\! z\!-\! z_d(t)\label{errors}
\end{align}
the projected on the horizontal plane distance error:
\begin{align}
e_d(t) =\sqrt{e_x^2(t)+e_y^2(t)}\label{error_dist}
\end{align}
as well as the projected on the horizontal plane orientation error:
\begin{align}
e_o(t) = \frac{e_x(t)}{e_d(t)} s_{\psi(t)}- \frac{e_y(t)}{e_d(t)}c_{\psi(t)}\label{error_orient}
\end{align}
where $s_\star=\sin(\star)$ and $c_\star=\cos(\star)$. It should be noted that the tracking control problem is solved if the projected on the horizontal plane distance error $e_d$, the vertical error $e_z$ and the orientation error $e_o$ reduce to zero. Moreover it should be noticed that the orientation error $e_o$ is well-defined only for nonzero values of $e_d$. For this respect, the proposed approach has been designed to further guarantee that $e_d(t)\geq \epsilon, \forall t\geq 0$, where $\epsilon$ is an arbitrarily small positive value, that avoids the aforementioned singularity issue when $e_d\rightarrow 0$. Therefore, in view of \eqref{error_dist} and considering the state constraint set $X$ of \eqref{eq4}, we can define a feasible error set given as:
\begin{align}
\mathcal{E}\coloneqq \{\bs{x}\in X :  \sqrt{e_x^2(t)+e_y^2(t)} \geq \epsilon,~ \text{with}~ \epsilon > 0\} \label{feasible_error_set}
\end{align}

Now, differentiating the aforementioned errors of \eqref{errors}-\eqref{error_orient} and employing \eqref{11}, we arrive at:
\begin{align}
\dot{e}_d 
& = \frac{e_x c_\psi + e_y s_\psi }{e_d} u - \frac{e_x \dot{x}_d + e_y \dot{y}_d }{e_d}+ \frac{e_x \omega_1 + e_y \omega_2 }{e_d}\label{eq.15} \\
\dot{e}_z & =w - \dot{z}_d + \omega_3 \label{eq.16} \\
\dot{e}_o  
& = \frac{e_y s_\psi c^2_\psi -e_x s^2_\psi c_\psi}{e_d^2} u +\frac{e_x e_y}{e_d^2} r +\frac{e_y \omega_1 - e_x \omega_2 }{e_d^2} s_\psi c_\psi\nonumber \\ 
&\hspace{5mm}  +\frac{\left(e_x \dot{y}_d -e_y \dot{x}_d\right) s_\psi c_\psi }{e_d^2}\label{eq.17}
\end{align}
By defining the error vector $\bs{e}=[e_d,~e_z,~e_o]^\top$, the aforementioned formulas can be written in matrix form as:
\begin{align}
\dot{\bs{e}}=J(\bs{e},\bs{p}_d)\bs{v}+ \zeta(\bs{e},\dot{\bs{p}}_d)+\xi(\bs{e},\bs{\omega})\label{uncertain_dynamic}
\end{align}
where:
\begin{align*}
&J(\bs{e},\bs{p}_d)\bs{v}
 \coloneqq
\begin{bmatrix}
\frac{e_x c_\psi + e_y s_\psi }{e_d} & 0 & 0 \\
0 & 1 & 0 \\
\frac{e_y s_\psi c^2_\psi-e_x s^2_\psi c_\psi}{e_d^2} & 0 & \frac{e_x e_y}{e_d^2} \nonumber\\
\end{bmatrix}
\begin{bmatrix}
u \\
w \\
r \\
\end{bmatrix} \\
&\!\!\zeta(\bs{e},\dot{\bs{p}}_d)\!\!\coloneqq\!\!
\begin{bmatrix}
-\frac{e_x \dot{x}_d + e_y \dot{y}_d }{e_d} \\
-\dot{z}_d \\
\frac{\left(e_x \dot{y}_d -e_y \dot{x}_d\right) s_\psi c_\psi }{e_d^2} \\
\end{bmatrix}\!,~\xi(\bs{e},\bs{\omega})\!\!\coloneqq\!\!\begin{bmatrix}
\frac{e_x \omega_1 + e_y \omega_2 }{e_d} \\
\omega_3 \\
\frac{e_y \omega_1 - e_x \omega_2 }{e_d^2} s_\psi c_\psi\! \\
\end{bmatrix}
\end{align*}
which are the uncertain error dynamics of the underwater vehicle system. The corresponding nominal error dynamics can be now given by:
\begin{align}
\dot{\bs{\hat{e}}}=J(\bs{\hat{e}},\bs{p}_d)\hat{\bs{v}}+ \zeta(\bs{\hat{e}},\dot{\bs{p}}_d)\label{nominal_dynamic}
\end{align}
{{It should be noticed that we use the ${\bs{\hat{e}}}$ notation for the nominal error state in order to account for the mismatch between the real error state and the nominal one which will be used in the following analysis.}}
\begin{remark}
It should be noted that the constraint set $\mathcal{E}$ in \eqref{feasible_error_set} guarantees that $J(\cdot)$ is non-singular. Thus, there exists a strictly positive constant $\underline{J}$ such that: $\lambda_{\min} \left(\frac{J(\cdot) +J^\top\!(\cdot)}{2}\right) \ge \underline{J} > 0$. 
\end{remark}
{{\begin{remark}\label{remark_2}
It can easily be shown that the function $\xi(\bs{e},\bs{\omega})$ appearing in \eqref{uncertain_dynamic} is a bounded function and there exists a strictly positive constant $\tilde{\xi}$ such that: $|| \xi(\bs{e},\bs{\omega}) || \leq \tilde{\xi}$, $\tilde{\xi} > 0$. In particular, by noting that $\frac{|e_x|}{\sqrt{e_x^2+e_y^2}} \le 1$, $\frac{|e_y|}{\sqrt{e_x^2+e_y^2}} \le 1$ and $|s_\psi| \le 1$, $|c_\psi| \le 1$, we deduce that $\tilde{\xi}$ does not depend of the bounds of $\bs{e}$. 
\end{remark}	}}
\subsection{State Feedback Design}
Consider the feedback law:
\begin{align}
\bs{v}= \bs{\hat{v}}(\bs{\hat{e}})+\kappa(\bs{e},\bs{\hat{e}}) \label{state_feedback}
\end{align}
which consists of a nominal control action $\bs{\hat{v}}(\bs{\hat{e}}) \in V$ and a state feedback law $ \kappa : \mathbb{R}^3 \times \mathbb{R}^3 \rightarrow V$.  The control action $\bs{\hat{v}}(\bs{\hat{e}})$ will be the outcome of a FHOCP solved for the nominal dynamics \eqref{nominal_dynamic} while the state feedback law $\kappa(\cdot)$ is designed in order to guarantee that the real trajectory $\bs{e}(t)$ (i.e., the solution of  \eqref{uncertain_dynamic}) always remain inside a bounded tube centered along the nominal trajectory $\bs{\hat{e}}(t)$ i.e., the solution of \eqref{nominal_dynamic}. Now let us define by $\bs{z}(t)=[z_1(t),~z_2(t),~z_3(t)]^\top$, and then the deviation between the real error state and the nominal one is given as:
\begin{align}
\bs{\rho}(t) \coloneqq \bs{e}(t) - \bs{\hat{e}}(t) \label{new_error}
\end{align}
\noindent with  $\bs{\rho}(0)=\bs{e}(0)-\bs{\hat{e}}(0)=\bs{0}$. In view of \eqref{new_error}, the dynamics of $\bs{\rho}(t)$ can be given as:
\begin{align}
\dot{\bs{\rho}} & = \dot{\bs{e}}-\dot{\hat{\bs{e}}} \notag \\
& = J(\bs{e}, \bs{p}_d) \bs{v} - J(\hat{\bs{e}}, \bs{p}_d) \hat{\bs{v}} + \zeta(\bs{e}, \dot{\bs{p}}_d)- \zeta(\hat{\bs{e}}, \dot{\bs{p}}_d) + \xi(\bs{e}, \bs{\omega}). \notag
\end{align}
By adding and subtracting the term $J(\bs{e}, \bs{p}_d)\hat{\bs{v}}$ and by defining the function $h(\bs{e}, \hat{\bs{v}}) \coloneqq J(\bs{e}, \bs{p}_d)\hat{\bs{v}}$, the latter becomes:
\begin{align}
\dot{\bs{\rho}}
= & \ h(\bs{e}, \hat{\bs{v}})- h(\hat{\bs{e}}, \hat{\bs{v}})+J(\bs{e}, \bs{p}_d) (\bs{v} - \hat{\bs{v}}) \notag \\ 
& + \zeta(\bs{e}, \dot{\bs{p}}_d)- \zeta(\hat{\bs{e}}, \dot{\bs{p}}_d) + \xi(\bs{e}, \bs{\omega}). \label{eq:z_dynamics}
\end{align}
Note that for the continuously differentiable functions $h(\cdot)$ and $\zeta(\cdot)$ the following hold:
\begin{subequations}
	\begin{align}
	\|h(\bs{e}, \hat{\bs{v}})- h(\hat{\bs{e}}, \hat{\bs{v}})\| & \le \mathcal{L}_1 \|\bs{e}-\hat{\bs{e}}\| = \mathcal{L}_1 \|\bs{\rho}\|, \label{eq:lip_1} \\
	\|\zeta(\bs{e}, \dot{\bs{p}}_d)- \zeta(\hat{\bs{e}}, \dot{\bs{p}}_d)\| & \le \mathcal{L}_2 \|\bs{e}-\hat{\bs{e}}\| = \mathcal{L}_2 \|\bs{\rho}\|, \label{eq:lip_2}
	\end{align}
\end{subequations}
where $\mathcal{L}_1$, $\mathcal{L}_2 > 0$ stand for their Lipschitz constants.
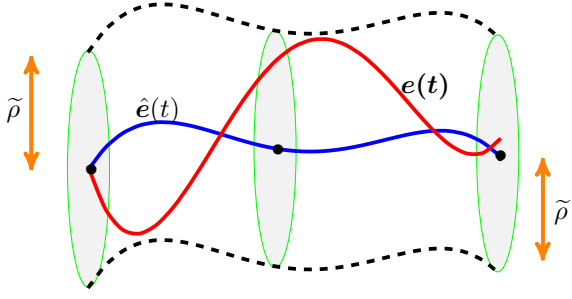
\begin{figure}[t!]
	\vspace{2mm}
	\centering
	\begin{tikzpicture}[scale = 0.8]
	\draw[color = green, fill=black!5] (-0.55,-0.48) ellipse (0.35cm and 1.97cm);
	\draw[color = green, fill=black!5] (2.55,-0.15) ellipse (0.35cm and 1.97cm);
	\draw[color = green, fill=black!5] (6.25,-0.22) ellipse (0.35cm and 1.97cm);
	
	\draw[scale=0.7,domain=-0.8:9,smooth,variable=\x,blue, line width = 0.05cm] plot ({\x},{-0.005*\x*\x*\x*\x+0.0868*\x*\x*\x-0.4554*\x*\x+0.6542*\x+0.1467});
	
	\draw[scale=0.7,domain=-0.8:9,smooth,variable=\x,red, line width = 0.05cm] plot ({\x},{0.93+0.0139*\x*\x*\x*\x-0.2491*\x*\x*\x+1.2240*\x*\x-0.7324*\x-3.0411});
	
	\draw[scale=0.7,domain=-0.8:9,smooth,variable=\x,black, line width = 0.05cm, dashed] plot ({\x},{-0.005*\x*\x*\x*\x+0.0868*\x*\x*\x-0.4554*\x*\x+0.6542*\x+0.1467+2.8});
	\draw[scale=0.7,domain=-0.8:9,smooth,variable=\x,black, line width = 0.05cm, dashed] plot ({\x},{-0.005*\x*\x*\x*\x+0.0868*\x*\x*\x-0.4554*\x*\x+0.6542*\x+0.1467-2.8});
	
	\draw (6.29, -0.3) node[circle, inner sep=0.8pt, fill=black, label={below:{$ $}}] (A1) {};
	
	\draw [color=orange,thick,->,>=stealth', line width = 0.5mm](-1.5, -0.5) to (-1.5, 1.4);
	\draw [color=orange,thick,->,>=stealth', line width = 0.5mm](-1.5, 1.4) to (-1.5, -0.5);
	\draw [color=orange,thick,->,>=stealth', line width = 0.5mm](7.0, -0.3) to (7.0, -2.0);
	\draw [color=orange,thick,->,>=stealth', line width = 0.5mm](7.0, -2.0) to (7.0, -0.3);
	\node at (-1.8, 0.5) {$\widetilde{\rho}$};
	\node at (7.3, -1.2) {$\widetilde{\rho}$};
	\node at (2.6,-0.16) {$\bullet$};
	\node at (6.3,-0.27) {$\bullet$};
	\node at (-0.5,-0.50) {$\bullet$};
	\node at (5.05, 0.9) {$\bs{e(t)}$};
	\node at (0.6, 0.5) {$\hat{\bs{e}}(t)$};
	\end{tikzpicture}
	\caption{The hyper-tube centered along the trajectory $\hat{\bs{e}}(t)$ (depicted by the blue line) with radius $\widetilde{\rho}$. Under the proposed control law, the real trajectory $\bs{e}(t)$ (depicted with red line) lies inside the hyper-tube for all times, i.e., $\|\bs{\rho}(t)\| \le \widetilde{\rho}$, $\forall t \in \mathbb{R}_{\ge 0}$.}\vspace{-3mm}
	\label{fig:tube}
\end{figure}
Now based on the aforementioned analysis the following Lemma can be stated: 
\begin{lemma} \label{lemma:tube}
The state feedback law designed by:
\begin{align} \label{eq:kappa_law}
\kappa(\bs{e}, \bs{\hat{e}}) \coloneqq - \sigma (\bs{e} - \bs{\hat{e}}),
\end{align}
where $\sigma$ is chosen such that:
\begin{align} \label{eq:sigma}
\sigma > \frac{\mathcal{L}_1+\mathcal{L}_2}{\underline{J}}.
\end{align}
renders the set:
\begin{align*}
\mathcal{P} \coloneqq \left\{\bs{\rho} \in \mathbb{R}^n: \|\bs{\rho}\| \le \widetilde{\rho} \right\},
\end{align*}
an RCI set for the error dynamics \eqref{eq:z_dynamics}, according to Definition \ref{def:RPI_set}, with:\vspace{-3mm}
\begin{align*}
\widetilde{\rho} \coloneqq \frac{\widetilde{\xi}}{\sigma \underline{J} - \mathcal{L}_1 - \mathcal{L}_2} > 0.
\end{align*}
\end{lemma}\vspace{2mm}
\textbf{Proof : } Consider the positive definite function $\Lambda(\bs{\rho}) = \frac{1}{2} \|\bs{\rho}\|^{2}$. The time derivative of $\Lambda$ along the trajectories of the system \eqref{eq:z_dynamics} is given by:
\begin{align}
\dot{\Lambda}(\bs{\rho}) = & \ \bs{\rho}^\top \dot{\bs{\rho}} \notag \\
= & \ \bs{\rho}^\top \left[h(\bs{e}, \hat{\bs{v}})- h(\hat{\bs{e}}, \hat{\bs{v}}) \right]+ \bs{\rho}^\top \left[ \zeta(\bs{e}, \dot{\bs{p}}_d)- \zeta(\hat{\bs{e}}, \dot{\bs{p}}_d)\right] \notag \\ 
& + \bs{\rho}^\top J(\bs{e}, \bs{p}_d) (\bs{v} - \hat{\bs{v}}) + \bs{\rho}^\top \xi(\bs{e}, \bs{\omega}). \notag
\end{align}
By employing \eqref{eq:lip_1}-\eqref{eq:lip_2} and the boundedness of $\| \xi(\cdot) \|$ given in Remark \ref{remark_2}, i.e, $\|\xi(\bs{e},\bs{\omega})\| \le \widetilde{\xi}$, $\widetilde{\xi} > 0$, the latter becomes:
\begin{align}
\dot{\Lambda}(\bs{\rho})\!\le\!(\mathcal{L}_1+\mathcal{L}_2) \|\bs{\rho}\|^2\!+\!\bs{\rho}^\top J(\bs{e}, \bs{p}_d) (\bs{v}\! -\! \hat{\bs{v}})\! +\!\widetilde{\xi}  \|\bs{\rho}\|. \label{eq:lyap_1}
\end{align}
Now, by taking $J = \frac{J+J^\top}{2}\!+\!\frac{J-J^\top}{2}$; using the fact that $y^\top\! \frac{J-J^\top}{2}y\! =\! 0$, $\forall y$; and substituting the control law \eqref{eq:kappa_law}, the inequality of \eqref{eq:lyap_1} yields:
\begin{align*}
\dot{\Lambda}(\bs{\rho}) \le \left[ - (\sigma \underline{J}-\mathcal{L}_1-\mathcal{L}_2) \|\bs{\rho}\| +\widetilde{\xi} \right]  \|\bs{\rho}\|.
\end{align*}
Thus, it holds that:
\begin{align}
\dot{\Lambda}(\bs{\rho}) < 0 \ {\rm when} \ \|\bs{\rho}\| > \frac{\widetilde{\xi}}{\sigma \underline{J} - \mathcal{L}_1 - \mathcal{L}_2}.  
\end{align}
Moreover, owing to the fact that $\bs{\rho}(0) = 0$, we have:
\begin{align*}
\|\bs{\rho}(t)\| \le \frac{\widetilde{\xi}}{\sigma \underline{J} - \mathcal{L}_1 - \mathcal{L}_2}, \  \forall t \ge 0.
\end{align*}
which concludes the proof. \hspace{40mm} $\square$

{{A graphical illustration of the proposed tube based control strategy is given in Fig \ref{fig:tube}. Under the proposed control scheme \eqref{state_feedback}, the real trajectory $\bs{e}(t)$ lies inside the hyper-tube which is centered along the nominal trajectory $\bs{\hat{e}}$ with radius $\widetilde{\rho}$ for all times, i.e., $\|\bs{\rho}(t)\| \le \widetilde{\rho}$, $\forall t \in \mathbb{R}_{\ge 0}$. }}

\subsection{Online Optimal Control}
As mentioned before, the control action $\bs{\hat{v}}(\bs{\hat{e}})$ in eq. \eqref{state_feedback} will be the outcome of a FHOCP solved for the nominal dynamics eq. \eqref{nominal_dynamic}. In this respect, consider a sequence of sampling times $\{t_k\}$, $k \in \mathbb{N}$, with a constant sampling period $0 < \delta_t < T$, where $T$ is a prediction horizon such that $t_{k+1} \coloneqq t_{k} + \delta_t$, $\forall k \in \mathbb{N}$. At each sampling time $t_k$, a FHOCP is solved as follows:
\begin{subequations}
	\begin{align}
	&\hspace{-4mm}\min\limits_{\hat{\bs{v}}(\cdot)} \left\{  \|\hat{\bs{e}}(t_k+T)\|^2_{\scriptscriptstyle P}\! +\! \int_{t_k}^{t_k+T}\!\!\! \Big[ \|\hat{\bs{e}}(\mathfrak{s})\|^2_{\scriptscriptstyle Q}\! +\!\|\hat{u}(\mathfrak{s})\|^2_{\scriptscriptstyle R} \Big] d\mathfrak{s} \right\} \hspace{0mm} \label{eq:mpc_cost_function} \hspace{-7mm}\\
	&\hspace{-3mm}\text{subject to:} \notag \\
	&\hspace{-3mm} \dot{\hat{\bs{e}}}(\mathfrak{s}) = J(\bs{\hat{e}}(\mathfrak{s}),\bs{p}_d)\hat{\bs{v}}+ \zeta(\bs{\hat{e}}(\mathfrak{s}),\dot{\bs{p}}_d), \ \ \hat{\bs{e}}(t_k) = \bs{e}(t_k), \label{eq:diff_mpc} \\
	&\hspace{-3mm} \hat{\bs{e}}(\mathfrak{s}) \in \overline{\mathcal{E}}, \ \ \hat{\bs{v}}(\mathfrak{s}) \in \overline{V},  \ \ \forall \mathfrak{\delta_t} \in [t_k,t_k+T], \label{eq:mpc_constrained_set} \\
	&\hspace{-3mm} \hat{\bs{e}}(t_k+T)\in \mathcal{F}, \label{eq:mpc_terminal_set}
	\end{align}
\end{subequations}
where $Q$, $P \in \mathbb{R}^{3 \times 3}$ and $R \in \mathbb{R}^{3 \times 3}$ are positive definite gain matrices. Moreover, $\overline{\mathcal{E}}$, $ \overline{V}$ and $\mathcal{F}$ are designing sets that are defined in order to guarantee that while the solution of FHOCP \eqref{eq:mpc_cost_function}-\eqref{eq:mpc_terminal_set} is derived for the nominal dynamics \eqref{nominal_dynamic}, the real trajectory $\bs{e}(t)$ and control inputs $\bs{v}(t)$ satisfy the corresponding state and input constraint sets  $\mathcal{E}$  and  $V$ respectively. More specifically, the following modification  is performed: 
\begin{align}
\overline{\mathcal{E}} \coloneqq \mathcal{E} \ominus \mathcal{P}, \ \ \overline{V} \coloneqq V \ominus \left[ -\sigma \circ \mathcal{P} \right].
\end{align}
This intuitively means that the sets $\mathcal{E}$, $V$ are tightened accordingly, in order to guarantee that while the nominal states $\hat{\bs{e}}$ and the nominal control input $\hat{\bs{v}}$ are calculated, the corresponding real error states $\bs{e}$ and real control input $\bs{v}$ satisfy the state and input constraints $\mathcal{E}$, $\mathcal{P}$ and $\mathcal{U}$, respectively\footnote{This constitutes a standard constraints set modification technique adopted in tube-based NMPC frameworks. For more details see \cite{yu_2013_tube}).}. Define the \emph{terminal set} by:
\begin{align} \label{eq:terminal_set_F}
\mathcal{F} \coloneqq \big\{\hat{\bs{e}} \in \overline{\mathcal{E}} : \|\hat{\bs{e}}\|_{\scriptscriptstyle P} \le \bar{\epsilon} \big\}, \ \  \bar{\epsilon}  > 0,
\end{align}
which is employed here in order to enforce the stability of the system \cite{frank_1998_quasi_infinite}. 
\subsubsection*{Newly Detected Obstacles}
as mentioned before, the obstacles within the workspace may be detected online by the vehicle's on--board sensors (e.g., multi--beam imaging or side scan sonar). In such a case, it should be assured that the solution of the FHOCP corresponds to the region that is accessible by the sensing capabilities of the vehicle.  This intuitively means it is required any potential new obstacles to be visible by the vehicle  even in the worst case (i.e., maximum velocity of the robot under maximum disturbances). Thus, assuming that $\bar{R}$ denotes the sensing range of the system as it is already stated in Section-\ref{Geometry}, the prediction
horizon $T$ should be set as follows:\vspace{-0mm}
\begin{align}
T\leq \frac{\bar{R}}{\bar{V}+\widetilde{\xi}}\label{new_obstacles}
\end{align} 	
\noindent where $\bar{V}$ is defined in \eqref{Vbar}.
\begin{remark}
It should be noticed that in a real scenario, AUVs use sonar sensors to obtain knowledge about the environment. The detection range of these sonar sensors (i.e., $\tilde{R}$) depends on many factors, including the frequency. Low frequency sonars can detect objects at very long distance, depending on the sound propagation environment.  Medium frequency sonars (typically operating between $7.5kHz$ and up to $30kHz$) can detect a object at a multiple nautical miles. On the other hands, high frequency sonars ($>100kHz$), typically used for underwater inspection can detect smaller objects at a few hundreds meters (i.e., $>100m$). On the other hand, as stated before, for needs of various common underwater tasks, the vehicle is required to move with relatively low speed. Thus, in view of \eqref{Vbar}, in a real scenario the predefined upper bound of the vehicle velocity can be tuned accordingly to the capability sensing range $\tilde{R}$ of the available sonar system in order to get an valuable prediction horizon enough for solving the FHOCP \eqref{eq:mpc_cost_function}-\eqref{eq:mpc_terminal_set}. 
\end{remark}
Now we are ready to state the main result of this work:
\begin{theorem}
Suppose that at time $t = 0$ the FHOCP \eqref{eq:mpc_cost_function}-\eqref{eq:mpc_terminal_set} is feasible. Then, the proposed feedback control law \eqref{state_feedback}, \eqref{eq:kappa_law}, renders the closed-loop system Input-to-State stable (ISS) with respect to the disturbances, for every initial condition $\hat{\bs{e}}(0)  \in \mathcal{E}$.
\end{theorem}
\begin{proof}
The proof of the theorem follows similar arguments presented in our previous work \cite{alex_IJRNC_2018}. Due to the fact that only the state of the nominal system is used while the FHOCP \eqref{eq:mpc_cost_function}-\eqref{eq:mpc_terminal_set} is solved, the on-line optimization does not depend on the disturbances. Thus, the \emph{feasibility proof} follows same arguments as in \cite{alex_ACC_2018, alex_IJRNC_2018,frank_1998_quasi_infinite} and is omitted here. Regarding the \emph{convergence analysis}, due to the fact that the set $\mathcal{P}$ is an RPI set, it holds that:
\begin{align} \label{eq:proof1}
\|\bs{\rho}(t)\| \le \widetilde{\rho}, \forall t \ge 0.
\end{align}
Due to the asymptotic stability of the nominal system, there exists a class $\mathcal{KL}$ function $\beta$ such that:
\begin{align} \label{eq:proof2}
\|\hat{\bs{e}}(t)\| \le \beta(\|\hat{\bs{e}}(0)\|, t), \forall t \ge 0.
\end{align}
By combining \eqref{new_error}, \eqref{eq:proof1} and \eqref{eq:proof2}, we get:
\begin{align*}
\|\bs{e}(t)\| \le \beta(\|\hat{\bs{e}}(0)\|, t) + \widetilde{\rho},
\end{align*}
for every $t \ge 0$, which leads to the conclusion of the proof.
\end{proof}
\section{Simulation Results}\label{SEC:SIM}
In this section we consider a simulation study in order to demonstrate the efficiency of the proposed approach. The simulation results were conducted using a dynamic simulation environment built in MATLAB \cite{heshmati2018robust, alex_shahab_ifac} with sampling time $0.1 sec$, which is common in a real time operation
with an underwater robotic system. We considered the tracking control problem for an underactuated AUV operating in a workspace including two obstacles which lie in $\bs{p}_{1}=[3,~0,~0]^\top$ and $\bs{p}_{2}=[-3,~0,~0]^\top$ respectively. The desired trajectory is a circle in horizontal plane defined by $\bs{p}_d(t) = [3\sin(\frac{\pi}{25}t) ,3\cos(\frac{\pi}{25}t) , 0]^T$. Notice that the desired trajectory that is required to be tracked by the AUV coincides with obstacles positions. The predefined upper bound of the vehicle velocities are defined as: $\bar{u}=0.4\frac{m}{s}$, $\bar{w}=0.3\frac{m}{s}$ and $\bar{r}=0.5\frac{rad}{s}$. Moreover, the capability sensing range and the horizon of the FHOCP are considered as $\bar{R}=1.5$ and $T=8*dt=0.8\sec$ respectively, satisfying the condition \eqref{new_obstacles}. Notice that the obstacles are detected 
and are considered by the controller when they are within the sensing range of the robot. The initial configuration is depicted in Fig \ref{fig2}.  In addition, in the subsequent simulation study the dynamics of the considered AUV were affected by external disturbances in the form of slowly time varying sea currents acting along $x$, $y$ and $z$ axes modeled by the corresponding velocities $\omega_1=0.1\sin(2\frac{\pi}{15}t)\frac{m}{s}$, $\omega_2=0.1\cos(2\frac{\pi}{15}t)\frac{m}{s}$ and $\omega_3=0.1\sin(2\frac{\pi}{15}t)\frac{m}{s}$. Furthermore, the parameter $\epsilon$ defined in \eqref{feasible_error_set} is set to $\epsilon=0.1$.
\begin{figure}[t!]
	\centering
	\includegraphics[scale=0.33]{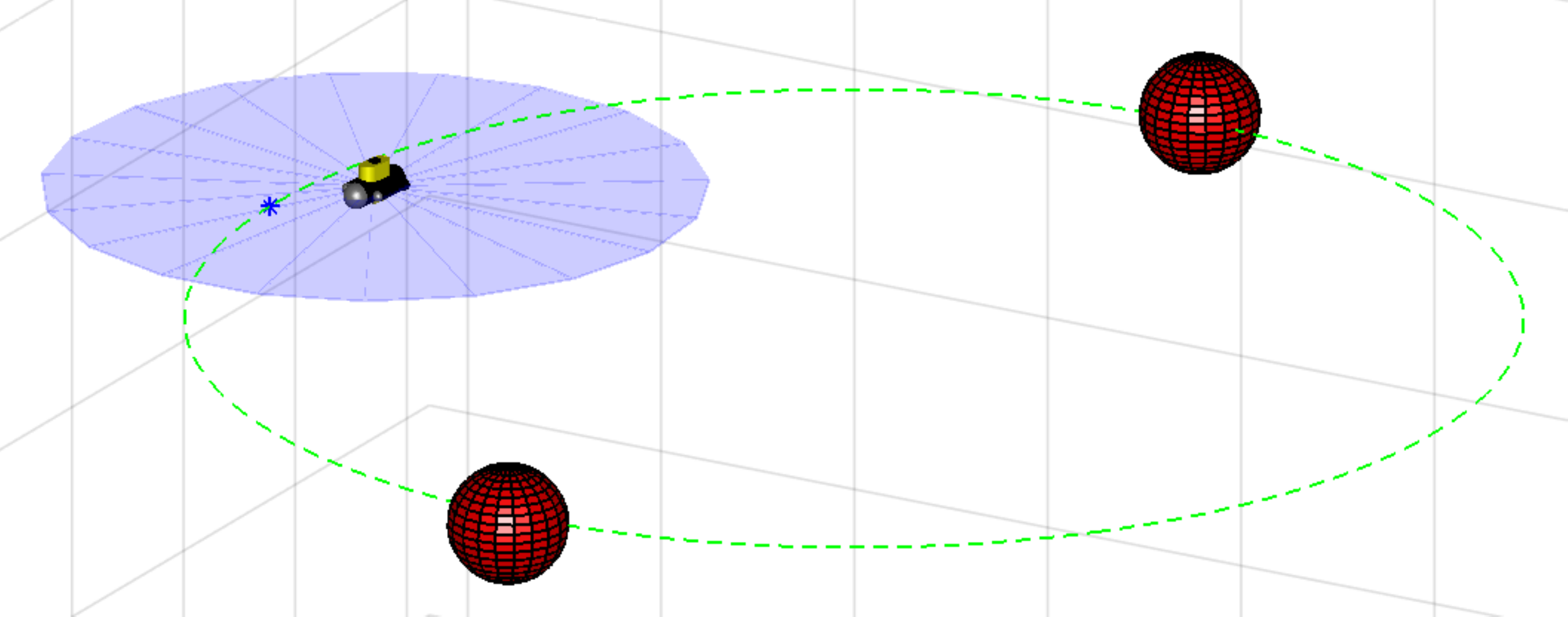}	
	\caption{Simulation setup: underactuated AUV operating in workspace including obstacles. The sensing range of the robot is indicated by a blue circle region around the robot. The desired trajectory is depicted with green dashed line which coincides with obstacles positions. The purpose of the controller is to track the real time evolution of the desired trajectory which is depicted by a blue star. The obstacles are detected and considered to the controller when they are within the sensing range $\bar{R}$.}\label{fig2}\vspace{-1mm}
\end{figure}
The simulation scenario has been conducted over a time period of $100 \sec$ in which the robot is required to perform two encirclements, which means that it should avoid the obstacles twice each. The results are given in Fig. \ref{fig3}-Fig. \ref{fig4}. More specifically the errors evolution is depicted in Fig. \ref{fig3}. It should be noted that the four peaks in $e_d$ and $e_z$ stand for the fact that the robot is leaving the desired trajectory in order to avoid the obstacles. In addition, it can be observed that the error $e_d$ remains always greater than $\epsilon=0.1$, i.e., $e_d(t)\geq \epsilon, \forall t \ge 0$. Moreover, the control inputs evolution is depicted in  Fig. \ref{fig4}. It can be witnessed that the control input constraints always satisfied.
\begin{figure}[t!]
	\centering
	\includegraphics[scale=0.5]{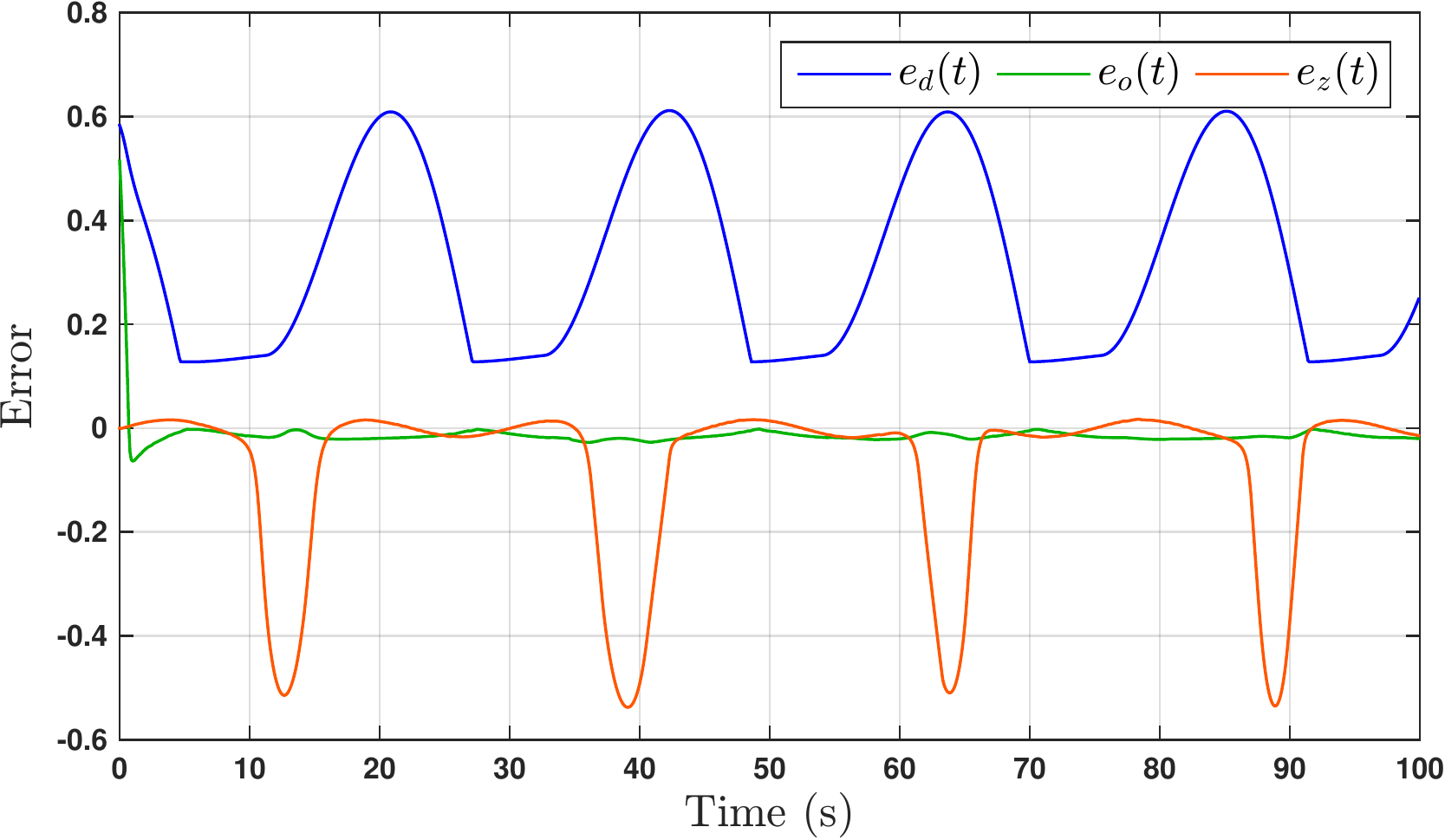}	
	\caption{The evolution of the error over time. The four peaks in $e_d$ and $e_z$ stand for the fact that the robot is leaving the desired trajectory in order to avoid the obstacles.}\label{fig3}\vspace{-4mm}
\end{figure}
\begin{figure}[t!]
	\centering
	\includegraphics[scale=0.46]{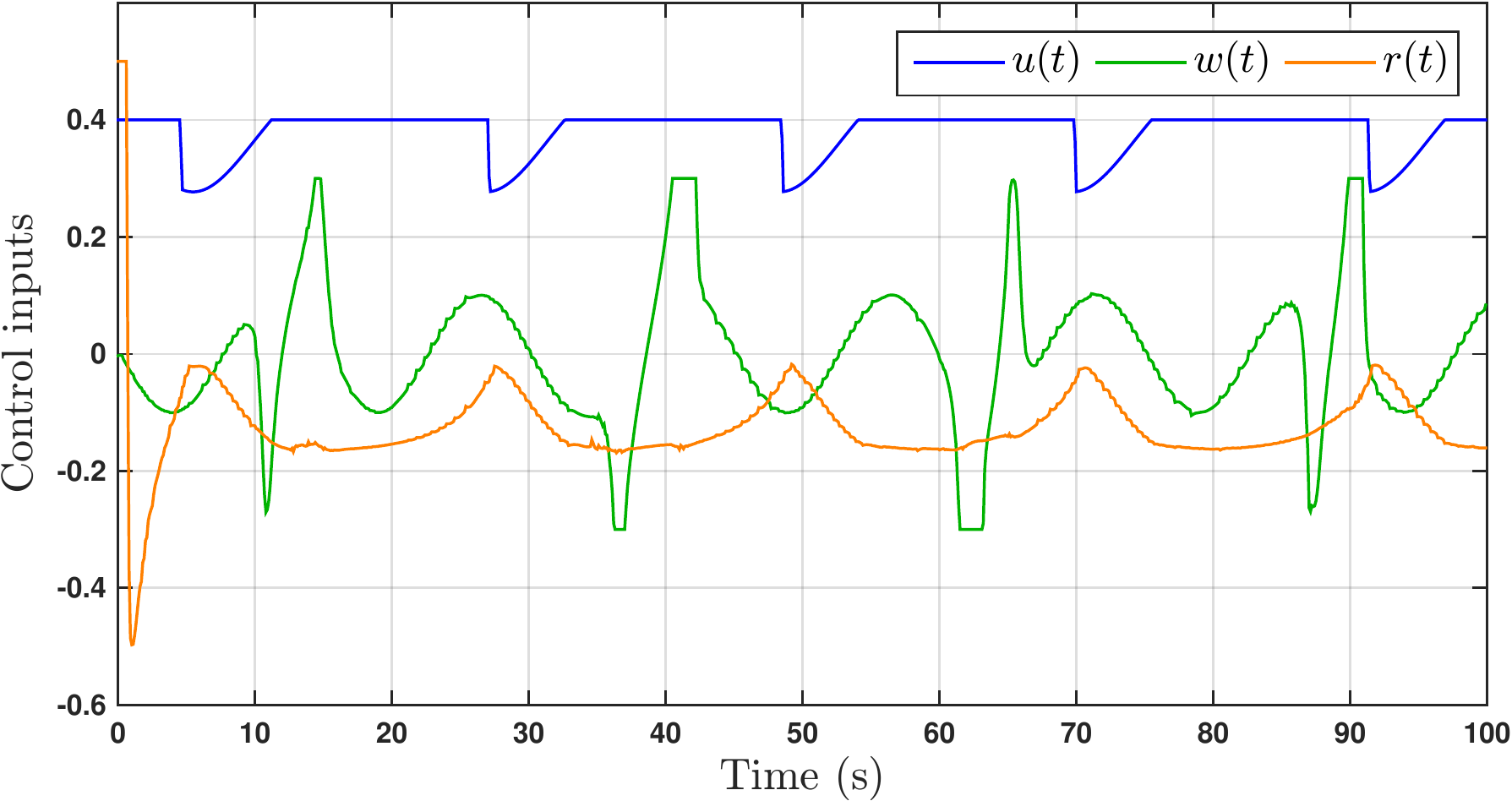}	
	\caption{The evolution of the control input signals over time.}\label{fig4}\vspace{-6mm}
\end{figure}

\noindent\textbf{Video:}\\
A video demonstrating the simulation of this section can be found in the following link:

\hspace{12mm}{\tt \small \href{https://youtu.be/jJfjnp-mw\_s}{https://youtu.be/jJfjnp-mw\_s}}

\section{Conclusions and Future work}\label{SEC:CON}
This paper presents a robust trajectory tracking control for underactuated Autonomous Underwater Vehciles operating in a constrained workspace including obstacles. The purpose of the controller is to steer the underactuated AUV on a desired trajectory inside a constrained and dynamic workspace. The workspace knowledge (i.e., obstacles positions) is constantly updated online via the vehicle’s sensors. Obstacle avoidance with any of the detected obstacles is guaranteed, despite the
presence of external disturbances. Moreover, various constraints such as: obstacles, workspace boundaries, predefined upper bound of the vehicle velocity (requirements for various underwater tasks such as seabed inspection, mosaicking etc.) are considered during the control design. The proposed feedback control law consists of two parts: i) a Finite Horizon Optimal Control Problem (FHOCP) and ii) a state feedback law which is tuned off-line and guarantees that the real trajectories remain inside a hyper-tube centered along the nominal
trajectories.  The closed-loop system has analytically guaranteed stability and convergence properties. Future research efforts will be devoted towards extending the proposed methodology for multiple Autonomous Underwater Vehicles operating in a dynamic environment including not only static but also moving obstacles.




\bibliographystyle{ieeetr}
\bibliography{ifacconf,mybibfileshahab,mybibfile}
\end{document}